\documentclass[11pt]{article}

\usepackage{bbm}
\usepackage{natbib}
\usepackage{amsmath}
\usepackage{graphicx} 
\usepackage{fullpage}
\usepackage{amsfonts}
\usepackage{amssymb}

\newcommand{\Poi}{\mathrm{Poi}}

\newcommand{\floor}[1]{{\left\lfloor #1 \right\rfloor}}

\newcommand{\bx}{{\bf x}}

\newcommand{\cB}{{\cal B}}
\newcommand{\cC}{{\cal C}}

\newcommand{\cN}{{\cal N}}

\newcommand{\cQ}{{\cal Q}}

\newcommand{\eps}{\epsilon}

\newcommand{\R}{{\mathbb R}}
\newcommand{\N}{{\mathbb N}}

\newcommand{\hy}{\hat{y}}

\newcommand{\E}{{\bf E}}

\newcommand{\Var}{{\bf Var}}

\renewcommand{\Pr}{{\bf Pr}}

\newtheorem{theorem}{Theorem}

\newtheorem{lemma}[theorem]{Lemma}

\newtheorem{definition}[theorem]{Definition}

\newcommand{\qed}{\hfill\rule{1ex}{1.5ex}}

\newenvironment{proof}{{\bf Proof:}}{\qed}

\title{
Failures of model-dependent generalization bounds \\
for least-norm interpolation \\
}
\usepackage{times}
\author{%
Peter L. Bartlett \\
UC Berkeley \\
peter@berkeley.edu \\
 \and
Philip M. Long \\
Google \\
plong@google.com \\
}

\date{}

\begin{document}

\maketitle

\begin{abstract}%
We consider bounds on the generalization performance of the least-norm
linear regressor, in the over-parameterized regime where it can
interpolate the data.  We describe a sense in which any generalization bound
of a type that is commonly proved in statistical learning theory must
sometimes be very loose when applied to analyze the least-norm interpolant.
In particular, for a variety of natural joint distributions on training examples,
any valid generalization bound that depends only on the output
of the learning algorithm, the number of training examples, and
the confidence parameter, and that satisfies a mild condition
(substantially weaker than monotonicity in sample size), must
sometimes be very loose---it can be bounded below by a constant
when the true excess risk goes to zero.
\end{abstract}

\section{Introduction}

Deep learning methodology has revealed some striking deficiencies of
classical statistical learning theory: large neural networks, trained
to zero empirical risk on noisy training data, have good predictive
accuracy on independent test data.  These methods are overfitting
(that is, fitting to the training data better than the noise should
allow), but the overfitting is benign (that is, prediction performance
is good). It is an important open problem to understand why this is
possible.

The presence of noise is key to why the success of interpolating
algorithms is mysterious.
Generalization of algorithms that produce a perfect fit in the
absence of noise has been studied for
decades (see~\citep{haussler1992} and its references).
A number of recent papers have provided generalization
bounds for interpolating algorithms in the absence of noise,
either for deep networks or in abstract frameworks motivated
by deep networks~\citep{ll-lonnsgdsd-18,arora2019fine,cao2019generalization,feldman2020does}.
The generalization bounds in these papers either do not
hold or become vacuous in the presence of noise:
Assumption A1 in~\citep{ll-lonnsgdsd-18} rules out noisy data;
the data-dependent bound in~\citet[Theorem~5.1]{arora2019fine}
becomes vacuous when independent noise is added to the $y_i$;
adding a constant level of independent noise to the
$y_i$ in~\citep[Theorem~3.3]{cao2019generalization}
gives an upper bound on excess risk that is at least a constant;
and the analysis in~\citep{feldman2020does} concerns the noise-free
case.

There has also been progress on bounding the gap between average loss
on the training set and expected loss on independent test data, based
on uniform convergence arguments that bound the complexity of
classes of real-valued functions computed by deep networks. For instance,
the results in~\citep{b-scpcnn-98} for sigmoid nonlinearities rely on
$\ell_1$-norm bounds on parameters throughout the network, and those
in~\citep{bft-snmbnn-17} for ReLUs rely on spectral norm bounds of the weight
matrices throughout the network (see also the refinements
in~\citep{bm-rgcrbsr-02,nts-nbccnn-15,bft-snmbnn-17,%
grs-siscnn-18,long2019size}. 
These bounds involve distribution-dependent function classes,
since they depend on some measure of the complexity
of the output model that may be expected to be small for
natural training data.
For instance, if some training method
gives weight matrices that all have small spectral norm, the bound
in~\citep{bft-snmbnn-17} will imply that the gap between empirical
risk and predictive accuracy will be small. But while it is possible
for these bounds to be small for networks that trade off fit to the
data with complexity in some way, it is not clear that a network that
interpolates noisy data could ever have small values of these
complexity measures.  This raises the question: are there any good
data-dependent bounds for interpolating networks?

\citet{zbhrv-udlrrg-17} claimed, based on empirical evidence,
that conventional learning theoretic tools are useless for
deep networks, but they
considered the case of a fixed class of functions defined, for
example, as the set of functions that can be computed by a neural
network, or that can be reached by stochastic gradient descent with
some training data, no matter how unlikely. 
These observations illustrate the need to consider distribution-dependent
notions of complexity in understanding the generalization performance
of deep networks. 
The study of such distribution-dependent complexity notions has a long
history in nonparametric statistics, where it is central to the
problem of model selection~\citep[see, for example,][and its
references]{bartlett2002model};
uniform convergence analysis over a
level in a complexity hierarchy is part of a standard outline for
analyzing model selection methods.

\citet{DBLP:conf/nips/NagarajanK19} provided an example of a scenario
where, with high probability, an algorithm generalizes well, but
two-sided uniform convergence fails for any hypothesis space that
is likely to contain the algorithm's output.  Their analysis takes
an important step in allowing distribution-dependent notions of
complexity, but only rules out the application of a specific set
of tools: uniform convergence over a model class of the absolute
differences between expectations and sample averages.  Indeed, in
their proof, the failure is an under-estimation of the accuracy of a
model---a model has good predictive accuracy, but performs poorly on
a sample (one obtained as a transformed but equally likely version of
the sample that was used to train the model).  However, in applying
uniform convergence tools to show good performance of an algorithm,
uniform bounds are only needed to show that bad models are unlikely
to perform well on the training data.  So if one wishes to provide
bounds that guarantee that an algorithm has {\em poor} predictive
accuracy, \citet{DBLP:conf/nips/NagarajanK19} provided an example
where uniform convergence tools will not suffice.  In contrast, we
are concerned with understanding what tools can provide guarantees
of {\em good} predictive accuracy of interpolating algorithms.

In this paper, motivated by the phenomenon of benign
overfitting in deep networks, we consider a simpler setting
where the phenomenon occurs, that of linear regression.
(Lemma 5.2 of~\citep{NDR20} adapts the construction of
\citep{DBLP:conf/nips/NagarajanK19} to show a similar failure
of uniform convergence in this context, and similarly cannot
shed light on tools that can or cannot guarantee good predictive
accuracy.)  We study the minimum norm linear interpolant. Earlier
work~\citep{bartlett2019benign} provides tight upper and lower
bounds on the excess risk of this interpolating prediction
rule under suitable conditions on the probability distribution
generating the data, showing that benign overfitting depends
on the pattern of eigenvalues of the population covariance
(and there is already a rich literature on related questions~\citep{lr-jikrrcg-18,BRT19,BHMM19,hastie2019surprises,hmrt-shdrlsi-19,NDR20,derezinski2019exact,li2020benign,tsigler2020benign}).
These risk bounds involve fine-grained
properties of the distribution.  Is this knowledge necessary? Is it
instead possible to obtain data-dependent bounds for interpolating
prediction rules?  Already the proof in~\citep{bartlett2019benign} provides
some clues that this might be difficult: when benign overfitting
occurs, the eigenvalues of the empirical covariance are a very poor
estimate of the true covariance eigenvalues---all but a small fraction
(the largest eigenvalues) are within a constant factor of each other.

In this paper, we show that in these settings there cannot be good
risk bounds based on data-dependent function classes
in a strong sense:
For linear regression with the minimum norm prediction rule, any
``bounded-antimonotonic''
model-dependent error bound that is valid for a sufficiently broad set
of probability distributions must be loose---too large by an additive
constant---for some (rather innocuous) probability distribution.
The bounded-antimonotonic condition 
formalizes the mild requirement that
the bound does not degrade very rapidly with additional data.
Aside from this constraint, our result applies for any
bound that is determined as a function of the output of the
learning algorithm,
the number of training examples, and the confidence parameter.
This function could depend on the level in a hierarchy of models
where the output of the algorithm lies. Our result applies 
whether the bound is obtained by uniform convergence over a 
level in the hierarchy, or in some other way.

The intuition behind our result is that benign overfitting can only
occur when the test distribution has a vanishing overlap with the
training data. Indeed, interpolating the data in the training sample
guarantees that the conditional expectation of the prediction rule's
loss on the training points that occur once must be at least the noise
level.  Using a Poissonization method, we show that a situation where
the training sample forms a significant fraction of the support of the
distribution is essentially indistinguishable from a benign
overfitting situation where the training sample has measure
zero. Since we want a data-dependent bound to be valid in both cases,
it must be loose in the second case.

\section{Preliminaries and main results}


We consider prediction problems with patterns $x\in\ell_2$ and labels
$y\in\R$, where $\ell_2$ is the space of square summable sequences of
real numbers. In fact, all probability distributions that we consider
have support restricted to a finite dimensional subspace of $\ell_2$,
which we identify with $\R^d$ for an appropriate $d$.
For a joint distribution $P$ over $\R^d \times \R$
and a hypothesis $h : \R^d \rightarrow \R$,
define the {\em risk} of $h$ to be
\[
R_P(h) = \E_{(x,y) \sim P} [ (y - h(x))^2].
\]
Let $R_P^*$ be the minimum of $R_P$ over measurable functions.

For any positive integer $k$, 
a distribution $D$ over $\R^k$ is sub-Gaussian with parameter
$\sigma$ if,
for any $u \in \R^k$, 
$\E_{x \sim D}[ \exp( u \cdot (x-\E x)) ] 
 \leq \exp\left( \frac{ \| u \|^2 \sigma^2}{2} \right)$.
A joint distribution $P$ over $\R^d \times \R$
has {\em unit scale} if $(X_1,...,X_d,Y) \sim P$ is
sub-Gaussian with parameter $1$.
It is {\em innocuous} if 
\begin{itemize}
\item it is unit scale,
\item the marginal on $(X_1,...,X_d)$ is Gaussian, and
\item the conditional
of $Y$ given $(X_1,...,X_d)$ is continuous.
 \end{itemize}

A {\em sample} is a finite multiset of elements of
$\R^d \times \R$.  
A {\em least-norm interpolation algorithm} takes as input
a sample, and outputs
$\theta \in \R^d$ that minimizes $\| \theta \|$ subject to
\[
\sum_i (\theta \cdot \bx_i - y_i)^2
      = \min_{\hy_1,...,\hy_n} \sum_i (\hy_i - y_i)^2.
\]
We will refer both to 
the parameter vector $\theta$ output by 
the least-norm interpolation algorithm and 
the function
$x \rightarrow \theta \cdot x$ parameterized by $\theta$
as the {\em least-norm interpolant}.

A function $\epsilon(h,n,\delta)$ mapping a 
hypothesis $h$, a sample size $n$
and a confidence $\delta$
to a positive real number
is 
a {\em uniform model-dependent bound 
for unit-scale distributions}
if, for all unit-scale joint distributions $P$ and all
sample sizes $n$,
with probability at least $1-\delta$ over the random
choice of $S \sim P^n$, we have
\[
R_P(h) - R_P^* \leq \epsilon(h,n,\delta).
\]
The bound $\epsilon$ is {\em $c$-bounded antimonotonic}
for $c\ge 1$ if for all $h$, $\delta$,
$n_1$ and $n_2$, if $n_2/2\le n_1\le n_2$ then
$\epsilon(h,n_2,\delta) \le c\epsilon(h,n_1,\delta)$.
This 
requires
that the bound cannot get too much worse too quickly
with more data.
If $\epsilon(h,\cdot,\delta)$ is monotone-decreasing for all $h$ and
$\delta$, then it is $1$-bounded antimonotonic.

A set $\cB \subseteq \N$ is
{\em $\beta$-dense} if
$\mathrm{liminf}_{N \rightarrow \infty} 
\frac{ |\cB \cap \{ 1,...,N \}| }{N} 
  \geq \beta.
$
Say that $\cB \subseteq \N$ is
{\em strongly $\beta$-dense beyond $n_0$} if, for all $s \in \N$
such that $s^2 \geq n_0$, 
\[
\frac{ |\cB \cap \{ s^2,...,(s+1)^2 - 1 \}| }{2 s + 1} 
  \geq \beta.
\]
(Notice that if a set is strongly
$\beta$-dense beyond $n_0$, then it is 
$\beta$-dense.)

The following is our main result.
\begin{theorem}
\label{t:main}
If $\epsilon$ is a bounded-antimonotonic, uniform model-dependent
bound for 
unit-scale distributions, then there 
are constants
$c_0,c_1,c_2,c_3,c_4 > 0$ and innocuous distributions $P_{1}, P_{2},\ldots$,
such that for all
 $0 < \delta < c_1$,
the least-norm interpolant $h$ satisfies, for all large enough $n$,
\[
  \Pr_{S\sim P_n^n}
  \left[ R_{P_n}(h) - R_{P_n}^* \leq c_0/\sqrt{n} \right]\ge 1-\delta
\]
but nonetheless,
the set of $n$ such that 
\[
\Pr_{S\sim P_n^n}
  \left[\epsilon(h,n,\delta) > c_2\right] \geq \frac{1}{2}
\]
is strongly $(1  - \frac{c_3}{\log(1/\delta)})$-dense beyond
$c_4\log(1/\delta)$.
\end{theorem}

\section{Proof of Theorem~\protect\ref{t:main}}

Our proof uses the following lemma~\citep{birch1963maximum} 
(see also~\citep[p. 216]{feller}
and~\citep{batu2000testing,DBLP:journals/jacm/BatuFRSW13}), which 
has become known as the
``Poissonization lemma''.
We use $\Poi(\lambda)$ to denote the Poisson distribution with mean
$\lambda$: For $t\sim\Poi(\lambda)$ and $k\ge 0$,
  \[
    \Pr[t=k] = \frac{\lambda^k e^{-\lambda}}{k!}.
  \]
\begin{lemma}
\label{l:poissonization}
If, for $t \sim \Poi(n)$, you throw $t$ balls independently uniformly at
random into $m$ bins,
\begin{itemize}
\item the numbers of balls falling into the bins are mutually independent, and
\item the number of balls falling in each bin is distributed as
$\Poi(n/m)$.
\end{itemize}
\end{lemma}

For each $n$, our proof uses three distributions:
$D_n$, $Q_n$ and $P_n$.  
The first, $D_n$, is used to define $Q_n$ and $P_n$;
it is chosen so that the least-norm interpolant
performs well on $D_n$.
The distribution $Q_n$ is defined so that
the least-norm interpolant performs poorly on $Q_n$.
The distribution $P_n$ is defined so that, when
the least-norm interpolant performs well on $D_n$, it also performs
well on $P_n$.
Crucially, the least norm interpolants that arise from $Q_n$ and $P_n$
are closely related.

For each $n$, the joint distribution $D_n$ on $(x,y)$-pairs is defined
as follows.  Let $s = \floor{\sqrt{n}}$, $N = s^2$, $d = N^2$.  Let
$\theta^*$ be an arbitrary unit-length vector.  
Let 
$\Sigma_s$ be an
arbitrary covariance matrix with eigenvalues
$\lambda_1=1/81,\lambda_2=\cdots=\lambda_d=1/d^2$.
The marginal of $D_n$ on $x$ is
then $\cN(0, \Sigma_s)$.  For each $x \in \R^d$, the
distribution of $y$ given $x$ is 
$\cN\left(\theta^* \cdot x,\frac{1}{81}\right)$.  
For $d\ge 9$, since $(x,y)$ is Gaussian,
$\| \Sigma_s \| \leq 1/81$, and the
variance of $y$ is $1/81$, each $D_n$ is innocuous.

For an absolute constant positive integer $b$,
we get $Q_n$ from $D_n$ 
through the following steps.
\begin{enumerate}
    \item Sample $(x_1,y_1),\ldots,(x_{bn},y_{bn})\sim D_n^{b n}$.
    \item Define $Q_n$ on $\R^d \times \R$ so that its marginal
    on $\R^d$ is uniform on $U = \{x_1,\ldots,x_{bn}\}$ and its conditional
    distribution of $Y|X$ is the same as $D_n$.
\end{enumerate}

\begin{definition}
For a sample $S$, the {\em compression}
of $S$, denoted by $C(S)$, is defined to be
\[
  C(S) = \left((u_1,v_1),\ldots,(u_k,v_k)\right),
\]
where $u_1,\ldots,u_k$ are the unique elements of
 $\{x_1,\ldots,x_n\}$, and, for each $i$,
$v_i$ is the average of $\{ y_j : 1\le j \le n,\, x_j = u_i \}$.
\end{definition}

For the least-norm interpolation algorithm $A$,
for any pair $S$ and $S'$ of samples such that
$C(S) = C(S')$, we have $A(S) = A(S')$.  (This is true
because the least-norm interpolant $A(S)$ is uniquely defined
by the equality constraints specified by the compression $C(S)$.)

So that a generalization bound often must apply to $Q_n$, we need
to show that it is likely to be unit scale.  The proof
of this lemma is in Appendix~\ref{a:unit}.
\begin{lemma}
\label{l:unit}
There is a positive constant $c_5$ such that,
for all large enough $n$,
with probability $1 - \frac{c_5}{n}$, 
$Q_n$ has unit scale.
\end{lemma}

We can show that the least-norm interpolant is bad for $Q_n$
by only considering the points in the support of $Q_n$
that the algorithm sees exactly once.
  \begin{lemma}\label{lemma:Qbad.poisson}
   For any constant $c > 0$, there are constants
   $c_6, c_7 > 0$ such that,
   for all sufficiently large $n$,
   almost surely for $Q_n$ chosen randomly as described above,
   if $t$ is chosen randomly according to $\Poi(c n)$ and $S$ consists of
   $t$ random draws from $Q_n$, then with probability at least 
    $1 -   e^{-c_6 n}$ over $t$ and $S$,
      \[
        \E_{(x,y) \sim Q_n}\left[(A(C(S))(x)-y)^2\right] - 
        \E_{(x,y) \sim Q_n}\left[(f^*(X)-Y)^2\right]
          \ge c_7
      \]
    where $f^*$ is the regression function for $D_n$ (and hence also
    for $Q_n$).
  \end{lemma}
\begin{proof}
Recall that
$U=\{x_1,\ldots,x_{bn}\}$ is the support of the marginal of $Q_n$
on the independent variables. With probability $1$, $U$ has
cardinality $b n$. Define $h = A(C(S))$.  If some $x\in U$ appears
exactly once in $S$, then $h(x)$ is a sample from the distribution
of $y$ given $x$ under $D_n$.  Thus, for such an $x$, the expected
quadratic loss of $h(x)$ on a test point is the squared difference
between two independent samples from this distribution, which is twice
its variance, i.e.\ twice the expected loss of $f^*$, which is $2
\times 1/2 = 1$.  On any $x\in U$, whether or not it was seen exactly
once in $S$, by definition, $f^*(x)$ minimizes the expected loss given
$x$.

Lemma~\ref{l:poissonization} shows that, conditioned on the random
choice of $Q_n$,
the numbers of times the various $x$ in $U$ in $S$ are mutually
independent and, 
the probability that
$x \in U$
is seen exactly once in $S$ is 
$\frac{c}{b} \exp\left(-\frac{c}{b}\right)
 \geq \frac{c e^{-c}}{b}$.
Applying a Chernoff bound (see, for example, Theorem 4.5 in~\cite{MU}),
the probability that fewer than 
$ce^{-c}n/2$
members of $U$ are seen exactly once in $S$ is at most 
$e^{-c_6 n}$ for an absolute constant $c_6$.
Thus if $U_1$ is the (random) subset of points in $U$ that were seen
exactly once, we have
\begin{align*}
& Q_n\left[(h(X)-Y)^2\right] 
    -
    Q_n\left[(f^*(X)-Y)^2\right] \\
 & = \sum_{x \in U} 
  Q_n\left[((h(X)-Y)^2 - (f^*(X)-Y)^2) 1_{X = x}\right]  \\
& \geq \sum_{x \in U_1} 
                        \E[ (f^*(X)-Y)^2 1_{X = x} ] \\
& = \frac{|U_1|}{bn}.
\end{align*}
Since, with probability $1 - e^{- c_6 n}$, 
$|U_1| \geq ce^{-c}n/2$, this completes the proof.
\end{proof}

\begin{definition}
\label{e:good.same_marginal}
Define $P_n$ as follows.
  \begin{enumerate}
    \item Set the marginal distribution of $P_n$ on $\R^d$ the same as
    that of $D_n$.
    \item To generate $Y$ given $X=x$ for $(X,Y) \sim P_n$,
    first sample a random variable $Z$ whose distribution is
    obtained by conditioning a draw from a Poisson with mean 
    $\frac{c}{b}$
    on the event that it is at least $1$, then
    sample $Z$ values $V_1,\ldots,V_{Z}$ from the conditional
    distribution $D_n(Y|X=x)$, and set
  $
        Y = \frac{1}{Z} \sum_{i=1}^Z V_j.
  $
  \end{enumerate}
\end{definition}
Note that, since $D_n$ has a
density, $x_1,...,x_r$ are almost surely distinct and hence
$S$ drawn from $P_n^r$ has $C(S)=S$ a.s.

The following lemma implies that the bounds for
$P_n$ tend to be as big as those for $Q_n$.
\begin{lemma}
\label{l:easy.by.hard}
Define $Q_n$ as above let $\cQ_n$ be the resulting distribution over the random
choice of $Q_n$.  Suppose $P_n$ is defined as in
Definition~\ref{e:good.same_marginal}.
Let $c > 0$ be an arbitrary constant.
Choose $S$ randomly by choosing $t \sim \Poi(c n)$, $Q_n\sim\cQ_n$ and
$S\sim Q_n^t$. Choose $T$ by choosing
$r \sim B\left(bn, 1 - \exp\left(-\frac{c}{b}\right)\right)$ and
$T\sim P_n^r$. 
Then $C(S)$ and $T$ have the same distribution.  In particular,
for all $\delta > 0$,
for any function $\psi$ of the least norm interpolant $h$,
a sample size $r$, and a confidence parameter $\delta$, we have
\[
\E_{t \sim \Poi(c n),Q_n \sim \cQ_n} 
 [\E_{S \sim Q_n^t} [\psi(h(S),|C(S)|,\delta)]]
=
\E_{r \sim B\left(bn, 1 - \exp\left(-\frac{c}{b}\right)\right)} 
       [\E_{T \sim P_n^r} [\psi(h(T),r,\delta)].
\]
\end{lemma}
\begin{proof}
Let $\cC$ be the probability distribution over training sets obtained
by picking $Q_n$ from $\cQ_n$, 
picking $t$ from $\Poi(cn)$,
picking $S$ from $Q_n^t$ and compressing it.
Let $C = C(S)$ be a random draw from $\cC$.  Let
$n_C$ be the number of examples in $C$.

We claim that $n_C$ is distributed as 
$B\left(bn, 1 - \exp\left(-\frac{c}{b}\right)\right)$.  
Conditioned on $Q_n$, 
and recalling that $U$ is the support of $Q_n$,
for any $x \in U$, 
Lemma~\ref{l:poissonization} implies that
for each $x \in U$, the probability that
$x$ is not seen is the probability, under a 
Poisson with mean $\frac{c}{b}$, of drawing a $0$.  Thus,
the probability that $x$ is seen is $1 - \exp\left(-\frac{c}{b}\right)$.
Since the numbers of times different $x$ are seen in $S$
are independent, the number seen is distributed as 
$B\left(bn, 1 - \exp\left(-\frac{c}{b}\right)\right)$.  

Now, for each $x \in U$, the event that
it is in $C(S)$ is the same as the event that
at appears at least once in $S$.  Thus, conditioned on
the event that $x$ appears in $S$, the number of
$y$ values 
that are used to compute the
$y$ value in $C(S)$ is distributed as a Poisson
with mean $\frac{c}{b}$, conditioned on having a value at least
$1$.  

Let $D_{n,X}$ be the marginal distribution of $D_n$ on the $x$'s.
If we make $n$ independent draws from $D_{n,X}$, and then
independently reject some of these examples, to get $n_C$ draws,
the resulting $n_C$ examples are independent.
(We could first randomly decide the number $n_C$ of
examples to keep, and then draw those independently
from $D_n$, and we would have the same distribution.)

The last two paragraphs together, along with the
definition of $Q_n$, imply that the distribution over 
$T$ obtained by sampling $r$ from 
$B(bn, 1 - e^{-c/b})$ and $T$ from $P_n^r$ is the same as the distribution
over $C$ obtained by sampling $Q_n$ from $\cQ_n$, 
$t$ from $\Poi(cn)$,
then sampling $S$ from $Q_n^t$ and compressing it.  
Thus, the distributions of $T$ and $C(S)$ are the same, and hence the
distributions of $(h(T),|T|)$ and $(h(S),|C(S)|)$ are the same,
because $h(S)=h(C(S))$.
\end{proof}

We will use the following bound on 
a tail of the Poisson distribution.
\begin{lemma}[\citep{canonne2017short}]
\label{l:poisson.tail}
For any $\lambda,\alpha>0$,
$\Pr_{r \sim \Poi(\lambda)} (r \geq (1+\alpha)\lambda) \leq \exp
\left(-\frac{\alpha^2}{2(1+\alpha)}\lambda\right)$.
\end{lemma}

Armed with these tools, we now show that
$\epsilon$ must often have a large value.
\begin{lemma}
\label{l:bound_big}
Then there are positive constants $c_1,c_2,c_3,c_4$ such that,
for all $0 < \delta < c_1$,
the set of $n$ such that
\[
\Pr_{S\sim P_n^n}
  \left[\epsilon(h,n,\delta) > c_2\right] \geq \frac{1}{2}
\]
is strongly $(1 - \frac{c_3}{\log(1/\delta)})$-dense beyond $c_4\log(1/\delta)$.
\end{lemma}
\begin{proof}
We will think of the natural numbers as
being divided into bins $[1,2),[2,4), [4,7),...$
%
Let us focus our attention on one bin:
$\{ s^2, ..., (s+1)^2 - 1 \}$.  
Let $n$ denote the center of the bin, $n = s^2 + s$,
so that $s \sim \sqrt{n}$.

For a constant $c_8 > 0$ and any $\delta > 0$, 
Lemma~\ref{l:easy.by.hard}
implies
\begin{equation}
\label{e:easy.by.hard}
\E_{r \sim  B(bn, 1 - e^{-c/b})} [ \Pr_{T \sim P_n^r} [ \eps(h(T),r,\delta) 
        \leq c_8 ]]
 = \E_{t \sim \Poi(c n),Q_n \sim \cQ_n} 
   [\Pr_{S \sim Q_n^t} [\eps(h(S),|C(S)|,\delta) \leq c_8]].
\end{equation}
Suppose that $\epsilon$ is $B'$-bounded-antimonotonic. Fix $B>0$ such
that
 $B>B'$. Then
\begin{align}
\nonumber
 & \E_{t \sim \Poi(c n),Q_n \sim \cQ_n} [ \Pr_{S \sim Q_n^t} [\eps(h,|C(S)|,\delta) \leq c_8]] \\
\nonumber
 & = \E_{t \sim \Poi(c n),Q_n \sim \cQ_n} [ \Pr_{S \sim Q_n^t} [B\eps(h,|C(S)|,\delta) \leq c_8B]] \\
\nonumber
 & \leq
\E_{t \sim \Poi(c n),Q_n \sim \cQ_n} [\Pr_{S \sim Q_n^t} [R_{Q_n}(h) - R_{Q_n}^* > 
B\eps(h,|C(S)|,\delta)]] \\
\label{e:by.regrets}
 & \hspace{1.0in}
 + \E_{t \sim \Poi(c n),Q_n \sim \cQ_n} [ \Pr_{S \sim Q_n^t} [
 R_{Q_n}(h) - R_{Q_n}^* \leq c_8B ]].
\end{align}

For each sample size $t$ and any $Q_n$ that has unit scale
\begin{align*}
& \Pr_{S \sim Q_n^t} [R_{Q_n}(h) - R_{Q_n}^* > B \eps(h,|C(S)|,\delta)] \\
 & \le 
\Pr_{S \sim Q_n^t} \left[R_{Q_n}(h) - R_{Q_n}^* > \eps(h,t,\delta)\right]
+ \Pr_{S \sim Q_n^t} \left[
    B\eps(h,|C(S)|,\delta)\le \eps(h,t,\delta)\right] \\
 & \leq \delta + \Pr_{S \sim Q_n^t} [|C(S)| < t/2]
\end{align*}
where the second inequality follows from the fact that
$\epsilon$ is a valid $B'$-bounded-antimonotonic uniform model-dependent bound
for unit-scale distributions and $B > B'$.  Combining this with
Lemma~\ref{l:unit}, we have
\[
\E_{t \sim \Poi(c n),Q_n \sim \cQ_n} [\Pr_{S \sim Q_n^t} [R_{Q_n}(h) - R_{Q_n}^* > 
B\eps(h,|C(S)|,\delta)]]
 \leq \delta + \frac{c_5}{n} + \Pr_{S \sim Q_n^t} [|C(S)| < t/2].
\]
Now by a union bound, for some constant $c_9>0$,
\begin{align}
\lefteqn{\E_{t \sim \Poi(c n),Q_n \sim \cQ_n} [\Pr_{S \sim Q_n^t}
[|C(S)| < t/2]]} & \notag\\
    & \le \E_{t \sim \Poi(c n),Q_n \sim \cQ_n} [\Pr_{S \sim Q_n^t}
        [|C(S)| < c_9n]] + \Pr_{t \sim \Poi(c n)} [t/2 \geq c_9n] \notag\\
    & = \Pr_{Z \sim B(bn, 1 - e^{-c/b})} [Z < c_9n]
        + \Pr_{t \sim \Poi(c n)} [t/2 \geq c_9n] \notag\\
    & \le \delta,   \label{ineq:chernoffpoisson}
\end{align}
where the last inequality follows from a Chernoff bound
and from Lemma~\ref{l:poisson.tail} with $\alpha=1-2c_9/c$,
provided $n=\Omega(\log(1/\delta))$ and provided we can choose $c_9$
to satisfy $c/2<c_9<b(1-e^{-c/b})$. Our choice of $b$ and $c$, 
specified below, will
ensure this.
In that case, we have that
\begin{align*}
 & \E_{t \sim \Poi(c n),Q_n \sim \cQ_n} [\Pr_{S \sim Q_n^t} [\eps(h,|C(S)|,\delta) 
  \leq c_8]] \\
 & \leq
 2 \delta
 + \frac{c_5}{n}
 + \E_{t \sim \Poi(c n),Q_n \sim \cQ_n} [\Pr_{S \sim Q_n^t}
 [R_{Q_n}(h) - R_{Q_n}^* \leq c_8B]].
\end{align*}
Applying Lemma~\ref{lemma:Qbad.poisson} to bound the RHS, if $n$ is large enough
and $c_8 B < c_7$, then
\[
\E_{t \sim \Poi(c n),Q_n \sim \cQ_n} [\Pr_{S \sim Q_n^t} [\eps(h,|C(S)|,\delta) \leq c_8]]
  \leq 3 \delta + \frac{c_5}{n}.
\]
Returning to (\ref{e:easy.by.hard}), we get
\begin{equation}
\label{e:ave.prob.small.random_r}
\E_{r \sim  B(bn, 1 - e^{-c/b})} [\Pr_{T \sim P_n^r} [ \eps(h,r,\delta))\leq c_8 ]]
   \leq 3 \delta + \frac{c_5}{n}.
\end{equation}
Let us now focus on the case that
$b = 2$ and $c = 2 \ln 2$, so that
\[
\E_{r \sim  B(bn, 1 - e^{-c/b})}[r] = (1 - e^{-c/b})b n = n.
\]
(And note that $c/2=\ln 2 < 1 = b(1-e^{-c/b})$, as required
for~\eqref{ineq:chernoffpoisson}.)
Chebyshev's inequality
implies
\[
\Pr_{r \sim  B(bn, 1 - e^{-c/b})}[ r \not\in [n - s, n + s ]]
   \leq c_{10}
\]
for an absolute positive constant $c_{10}$.
Returning now to (\ref{e:ave.prob.small.random_r}),
Markov's inequality implies
\begin{equation}
\label{e:most.bad.binomial}
\Pr_{r \sim  B(bn, 1 - e^{-c/b})} 
  [\Pr_{T \sim P_n^r} [ \eps(h,r,\delta))\leq c_8 ]
   > 1/2]
   \leq c_{11} \left( \delta+ \frac{1}{n} \right).
\end{equation}
Further, it is known~\citep{slud1977distribution,box1978statistics}
that there is an
absolute constant $c_{12}$ such that, for all large
enough $n$ and all $r_0 \in [n - s, n + s ]$,
\[
\Pr_{r \sim  B(bn, 1 - e^{-c/b})}[ r = r_0 ]
 \geq \frac{c_{12}}{\sqrt{n}}.
\]
Combining this with (\ref{e:most.bad.binomial})
and recalling that $s$ and $s'$ are $\Theta(\sqrt{n})$, we get
\[
\frac{| \{ r \in [n - s, n + s] : 
     \Pr_{T \sim P_n^r} [ \eps(h,r,\delta))\leq c_8 ] > 1/2
      \} |}{2 s + 1} 
\leq c_{13} \left( \delta+ \frac{1}{n} \right)
\leq \frac{c_{14}}{\log(1/\delta)},
\]
for $n \geq c_4\log(1/\delta)$ and small enough $\delta$.
Since, for all $r \in [n - s, n + s]$, we have
$P_r = P_n$, this completes the proof.
\end{proof}

The following bound can be obtained
through direction application of the results in
\citep{bartlett2019benign}.  The details are
given in Appendix~\ref{a:P.interpolant_succeeds}.
\begin{lemma}
\label{l:P.interpolant_succeeds}
There is a constant $c$ such that, for all large enough
$n$, with probability at least $1-\delta$, for $S \sim P_n^n$,
the least-norm interpolant $h$ satisfies
$R_{P_n}(h) - R^*_{P_n} \leq c \sqrt{\frac{\log(1/\delta)}{n}}$.
\end{lemma}

Combining this with Lemma~\ref{l:bound_big} proves
Theorem~\ref{t:main}.


\section{Acknowledgements}

We thank Andrea Montanari for alerting us to a flaw in an earlier
version of this paper.
We also thank Vaishnavh Nagarajan and Zico Kolter for helpful comments on
an earlier draft of this paper, and Dan Roy for
calling our attention to Lemma~5.2 of \citep{NDR20}.

\bibliographystyle{plainnat}
\bibliography{bib}

\appendix

\section{Proof of Lemma~\protect\ref{l:unit}}
\label{a:unit}

To prove Lemma~\ref{l:unit}, we will need some lemmas.  
The first is from \citep{buldygin1980sub}
(see \citep{rivasplata2012subgaussian}).
\begin{lemma}
\label{l:sum.scalars}
If $X_1$ is a sub-Gaussian random variable with parameter
$\sigma_1$, and $X_2$ is a (not necessarily independent)
sub-Gaussian random variable with parameter
$\sigma_2$, then $X_1 + X_2$ is sub-Gaussian 
with parameter $\sigma_1 + \sigma_2$.
\end{lemma}
This immediately implies the following.
\begin{lemma}
\label{l:sum.vectors}
If $X_1$ is a sub-Gaussian random vector with parameter
$\sigma_1$, and $X_2$ is a (not necessarily independent)
sub-Gaussian random vector with parameter
$\sigma_2$, then $X_1 + X_2$ is sub-Gaussian 
with parameter $\sigma_1 + \sigma_2$.
\end{lemma}
\begin{proof}
Any projection of $X_1 + X_2$ is the sum
of the projections of $X_1$ and $X_2$, 
so this
follows from Lemma~\ref{l:sum.scalars}.
\end{proof}

\begin{lemma}
\label{l:parts}
For a random vector $X = (X_1,...,X_k)$, 
if $X_1$ is sub-Gaussian with parameter $1/3$,
$X_2$ is sub-Gaussian with parameter $1/3$, and
$(X_3,...,X_k)$ is sub-Gaussian with 
parameter $1/3$, then $X$ is sub-Gaussian
with parameter $1$.
\end{lemma}
\begin{proof}
Embedding a random vector into a higher-dimensional space
by adding components that always evaluate to zero does
not affect whether it is sub-Gaussian, or its
sub-Gaussian parameter.  Since
\[
X = (X_1,0,...,0) + (0, X_2,0,...,0) + (0,0,X_3,...,X_k),
\]
applying Lemma~\ref{l:sum.vectors} above completes the proof.
\end{proof}

Now, for $U \sim D_n^m$, the uniform distribution
$Q$ over $U$, and $(X_1,...,X_d,Y) \sim Q$, we now
would like to show that $X_1$ is sub-Gaussian
with parameter $1/3$.  We will use the
following known sufficient condition, which can
be recovered by tracing through the constants 
in the proof of Proposition 2.5.2 of
\cite{vershynin2018high}.
\begin{lemma}
\label{l:exp.sq}
If a random variable $X$ satisfies
$\E\left[\exp\left(\frac{18 X^2}{e} \right)\right] \leq 2$, then $X$ 
is sub-Gaussian with parameter $1/3$.
\end{lemma}
Now we are ready to analyze the marginal distribution of the
first component.
\begin{lemma}
\label{l:gaussian.empirical.sub-gaussian}
For $U$ obtained from $m$ independent samples from
$\cN(0,\sigma^2)$ for $\sigma \leq 1/9$
if $Q$ is the uniform distribution
over $U$, then, with probability at least
$1 - \frac{3}{m}$, $Q$ is sub-Gaussian with parameter
$1/3$.
\end{lemma}
\begin{proof}
Define $a=18/e$ and let $Z = \E_{x \sim Q} [\exp(a x^2)]$.

We have
\begin{align*}
\E_{S \sim \cN(0,\sigma)^m} [ Z ]
& = \E_{S \sim \cN(0,\sigma)^m} [ \E_{x \sim Q} [\exp(a x^2)]] \\
& = \E_{x \sim \cN(0,\sigma)} [\exp(a x^2)] \\
& = \frac{1}{\sqrt{2 \pi} \sigma} \int_{-\infty}^{\infty}
  e^{a x^2} \exp\left(-\frac{x^2}{2 \sigma^2}\right) \; dx \\
& = \frac{1}{\sqrt{2 \pi} \sigma} 
\int_{-\infty}^{\infty}
  \exp\left(-\left( \frac{1}{2 \sigma^2} - a \right)  x^2 \right) \; dx \\
& = \frac{1}{\sqrt{2 \pi} \sigma} 
\times \frac{\sqrt{\pi}}{\sqrt{\frac{1}{2 \sigma^2} - a}} \\
& = \frac{1}{\sqrt{1 - 2a \sigma^2}}. \\
\end{align*}
Similarly
\begin{align*}
\Var_S [ Z ]
& = \Var_S [ \E_{x \sim Q} [\exp(a x^2)]] \\
& = \frac{1}{m} \Var_{x \sim \cN(0,\sigma)} [\exp(a x^2)] \\
& \leq \frac{1}{m} \E_{x \sim \cN(0,\sigma)} [\exp(2a x^2)] \\
& = \frac{1}{m \sqrt{1 - 4a \sigma^2}}. \\
\end{align*}
By Chebyshev's inequality, 
\begin{align*}
\Pr\left[Z \geq 
\frac{1}{\sqrt{1 - 2a \sigma^2}}
+ \frac{1}{\sqrt{3} (1 - 4a \sigma^2)^{1/4}}\right]
  \leq \frac{3}{m}.
\end{align*}
For $\sigma \leq 1/9$, recalling that $a=18/e$ shows that
$\Pr\left[Z \geq 2 \right] \leq \frac{3}{m}$
and applying Lemma~\ref{l:exp.sq} completes the proof.
\end{proof}

\medskip

Armed with these lemmas, we are now ready to prove
Lemma~\ref{l:unit}.
For $S \sim D_n^m$, let $Q$ be the uniform
over $S$.  For $(X_1,...,X_d,Y) \sim Q$,
Lemma~\ref{l:gaussian.empirical.sub-gaussian}
implies that, with probability $1 - 6/m$, 
$X_1$ and $Y$ are both sub-Gaussian with parameter
$1/3$.  It remains to analyze $(X_2,...,X_d)$.  
Let $S'$ be the projections
of the elements of $S$ onto these coordinates.  With probability
at least $1 - 3/m$, for all $s' \in S'$,
$|| s' || \leq \log(em^2/3)/\sqrt{d}$;
see~\citep[Lemma~5.17]{lovasz-vempala-07}.  Recalling
that $d = \Theta(n^2)$, if $m = b n$, then, for all large
enough $n$, with probability 
$1 - 3/m$, $\max_{s' \in S'} || s || \leq 1/6$, which implies that
$(X_2,...,X_d)$ is sub-Gaussian with parameter $1/3$.  Putting
this together with the analysis of $X_1$ and $Y$, and applying
Lemma~\ref{l:parts}, completes the proof.

\section{Proof of Lemma~\protect\ref{l:P.interpolant_succeeds}}
\label{a:P.interpolant_succeeds}

The lemma follows from Theorem 1 of~\citep{bartlett2019benign};
before showing how to apply it, let us first restate a special
case of the theorem for easy reference.

\subsection{A useful upper bound}
\label{s:bllt}

The special case concerns the least-norm interpolant applied
to training data $(x_1,y_1),...,(x_n,y_n)$ drawn from a
joint distribution $P$ over $(x,y)$ pairs.  The marginal
distribution of $x$ is Gaussian with
covariance
$\Sigma$.  
There is a unit length $\theta^*$ such that,
for all $x$, the conditional distribution of $y$ given $x$ 
has mean $\theta^* \cdot x$
is 
sub-gaussian with parameter $1$
and variance at most $1$.

We will apply an
upper bound in terms of the eigenvalues 
$\lambda_1 \geq \lambda_2 \geq ...$ of $\Sigma$.  
The bound
is in terms of two notions of the effective rank of
the tail of this spectrum:
\[
r_k(\Sigma) = \frac{\sum_{i > k} \lambda_i}{\lambda_{k+1}},\;
R_k(\Sigma) = \frac{\left(\sum_{i > k} \lambda_i \right)^2}
                   {\sum_{i > k} \lambda_i^2}.
\]
The rank of $\Sigma$ is assumed to be greater than $n$.
\begin{lemma}
\label{l:bllt}
There are $b,c, c_1 > 1$ for which the following holds.
For all $n$, $P$ and $\Sigma$ defined as above,
write $k^* = \min\{ k \geq 0: r_k(\Sigma) \geq b n \}$.
Suppose that $\delta < 1$ with $\log(1/\delta) < n/c$.
If $k^* < n/c_1$, then, with probability at least
$1-\delta$, the least-norm interpolant 
$h$ satisfies
\[
R_{P}(h) - R_{P}^* \leq
c \left( \max\left\{ 
     \sqrt{ \frac{ r_0(\Sigma) }{n} }.
     \frac{ r_0(\Sigma) }{n}.
     \sqrt{ \frac{ \log(1/\delta) }{n} }
         \right\}
       + \log(1/\delta) \left( \frac{k^*}{n} + \frac{n}{R_{k^*}(\Sigma)} \right)
   \right).
\]
\end{lemma}

\subsection{The proof}

To prove Lemma~\protect\ref{l:P.interpolant_succeeds},
we need to show that $P_n$ satisfies the requirements on $P$ in
Lemma~\ref{l:bllt}, and evaluate the effective ranks $r_k$ and $R_k$
of $P_n$'s covariance $\Sigma_s$.
Define $\alpha=1/d^2$. We have
\[
r_0 = \frac{1/81 + (d-1)\alpha}{1/81}
   = 1 + 81 (d-1)\alpha
\]
(which is bounded by a constant)
and
\[
R_0 = \frac{(1/81 + (d-1)\alpha)^2}{1/81^2 + (d-1)\alpha^2}.
\]
For $k > 0$, 
\[
r_k = R_k = d-k.
\]

Since $d$ grows faster than $n$, for large enough $n$, 
$k^*:=\min\left\{k:r_k\ge bn\right\} = 1$.  So 
\[
R_{k^*} = d - 1 = \Omega(n^2).
\]

Each sample from the distribution of
$Y$ given $X=x$ has a mean of $\theta^* \cdot x$, and is
sub-Gaussian with parameter at most $\frac{1}{9}$, 
and with variance at most $1/81$ 
(because increasing
$Z$ only decreases the variance of $Y$).  

Evaluating
Lemma~\ref{l:bllt} on $P_n$ then gives Lemma~\ref{l:P.interpolant_succeeds}.

\end{document}